\documentclass[letterpaper]{article} 
\usepackage{aaai25}  
\usepackage{times}  
\usepackage{helvet}  
\usepackage{courier}  
\usepackage[hyphens]{url}  
\usepackage{graphicx} 
\usepackage{natbib}  
\usepackage{caption} 
\usepackage{algorithm}
\usepackage{algorithmic}

\usepackage{newfloat}
\usepackage{listings}
\DeclareCaptionStyle{ruled}{labelfont=normalfont,labelsep=colon,strut=off} 
\floatstyle{ruled}
\newfloat{listing}{tb}{lst}{}
\floatname{listing}{Listing}
\usepackage[utf8]{inputenc} 
\usepackage{booktabs}       
\usepackage{amsfonts}       
\usepackage{nicefrac}       
\usepackage{microtype}      
\usepackage{xcolor}         
\usepackage{amsmath}
\usepackage[mathscr]{euscript}
\usepackage{amsthm}
\usepackage{amssymb}
\usepackage{bm}
\usepackage{subcaption}
\usepackage{mathtools} 
\usepackage{multirow}

\usepackage{thmtools}
\usepackage{thm-restate}
\usepackage{cleveref}

\newcommand{\coma}[1]{a_{#1}^{\texttt{c}}}
\newcommand{\comQ}[1]{Q_{#1}^{\texttt{c}}}
\newcommand{\jcomQ}{Q^{\texttt{c}}}

\newcommand{\comr}{r^{\texttt{c}}}
\newcommand{\jcoma}{\bm a^{\texttt{c}}}
\newcommand{\senderfor}[1]{\nu_{#1}}
\newcommand{\sname}[1]{$\mathtt{#1}$}

\newcommand{\shorte}{\textup{\texttt{=}}}

\newcommand{\ie}{\textit{i}.\textit{e}.}

\newtheorem{theorem}{Theorem}

\newtheorem{definition}[theorem]{Definition}
\newtheorem{assumption}[theorem]{Assumption}
\newtheorem{remark}[theorem]{Remark}

\title{The Bandit Whisperer: Communication Learning for Restless Bandits}
\author{
    Yunfan Zhao\textsuperscript{\rm 1, 2}\equalcontrib\thanks{Correspondence to: Tonghan Wang \textlangle twang1@g.harvard.edu\textrangle}
    Tonghan Wang\textsuperscript{\rm 1}\equalcontrib,
    Dheeraj Mysore Nagaraj\textsuperscript{\rm 3},
    Aparna Taneja\textsuperscript{\rm 3},
    Milind Tambe\textsuperscript{\rm 1, 3}
}
\affiliations{
    \textsuperscript{\rm 1}Harvard University,\\
    \textsuperscript{\rm 2}GE Healthcare,\\
    \textsuperscript{\rm 3}Google Deepmind}

\usepackage{bibentry}
\begin{document}

\maketitle

\begin{abstract}
Applying Reinforcement Learning (RL) to Restless Multi-Arm Bandits (RMABs) offers a promising avenue for addressing allocation problems with resource constraints and temporal dynamics. However, classic RMAB models largely overlook the challenges of (systematic) data errors, a common occurrence in real-world scenarios due to factors like varying data collection protocols and intentional noise for differential privacy. We demonstrate that conventional RL algorithms used to train RMABs can struggle to perform well in such settings. To solve this problem, we propose the first communication learning approach in RMABs, where we study which arms, when involved in communication, are most effective in mitigating the influence of such systematic data errors. 
In our setup, the arms receive Q-function parameters from similar arms as messages to guide behavioral policies, steering Q-function updates. We learn communication strategies by considering the joint utility of messages across all pairs of arms and using a Q-network architecture that decomposes the joint utility. Both theoretical and empirical evidence validate the effectiveness of our method in significantly improving RMAB performance across diverse problems.
\end{abstract}


\section{Introduction}

Restless Multi-Arm Bandits (RMABs) have been successfully applied to a range of  multi-agent constrained resource allocation problems, such as healthcare, online advertising, and anti-poaching~\cite{ruiz2020multi,lee2019optimal,hodge2015asymptotic,mate2022field,modi2019transfer,zhao2007myopic,bagheri2015restless,tripathi2019whittle}.  
In classical RMABs, there are a fixed number of arms, each may be described as an MDP
with two actions (active and passive). At each timestep, constrained by the global budget, the RMAB strategically selects a subset of arms to play the active action with the objective of maximizing long-term accumulated rewards. 
Although finding an exact RMAB solution is PSPACE-hard~\cite{papadimitriou1994complexity}, the whittle index solution is known to be near optimal~\cite{whittle1988restless,weber1990index} in this situation. Beyond classical RMABs, important applications of RMABs involve more complex systems with multiple actions, continuous state spaces, and arms opting in and out of the system.  Recent studies have explored a variety of methods to address these challenges~\cite{hawkins2003langrangian,killian2021beyond,verloop2016asymptotically,zayas2019asymptotically,ghosh2023indexability}, with deep Reinforcement Learning (RL) emerging as a particularly effective approach, achieving state-of-the-art results across various RMAB settings ~\cite{zhao2023towards,xiong2023finite,wang2023scalable,killian2022restless}.

Unfortunately, these previous studies consider RMAB models where the reward data can be collected from every arm reliably. In real-world scenarios  modeled via RMABs, such as maternal healthcare~\cite{biswas2021learn,vermagroup, behari2024decision} and SIS epidemic model concerning COVID intervention~\cite{yaesoubi2011generalized}, each arm represents a different entity, which could be a beneficiary or a geographical region, where data collection protocols are potentially different. Different arms may be affected by noise caused by
\textbf{systematic errors} that arise due to the following practical reasons. \textbf{(1) Variability in data collection and handling protocols.}  There may be differences in data collection in different arms (entities), e.g., differences in the frequency and quality of data collected, in worker awareness or standards for data preprocessing, 
cleaning, and digitization. This variability represents sampling bias~\cite{paulus2023reinforcing} and may cause systematic errors~\cite{dubrow2022local} in some arms. For example, in maternal care, specific data collection approaches may cause systematic overestimation of expected delivery dates~\cite{fulcher2020errors}.
\textbf{(2) Data integrity and privacy.} For example, when data is gathered through surveys, topical or political bias~\cite{paulus2023reinforcing} can introduce systematic errors in some arms due to influence of authorities or impacted populations
(e.g., keeping the number of COVID cases low \cite{paulus2023reinforcing}). Additionally, driven by the principle of differential privacy, noise might be intentionally added to the data~\cite{dwork2014algorithmic}.

Motivated by these scenarios, in this paper we study the RMAB model where some arms are affected by noise in their observed rewards at some states, with these noisy arms and rewards not known a priori. The challenge of applying existing RL methods to this setting is that
systematic errors might create false optima, preventing the exploration of states with true high rewards and leading to suboptimal actions selected by the RMAB. For example, the aforementioned systematic overestimation errors of women's delivery dates can lead to suboptimal resource allocation, thereby reducing the number of deliveries in healthcare facilities~\cite{fulcher2020errors}. 

We propose the first communication learning approach in RMABs 
to deal with systematic reward errors, enabling arms to learn to communicate useful information to improve local Q-learning. We formally model the communication learning problem as a Multi-Agent MDP~\cite{boutilier1996planning}, which is in addition to and distinct from per-arm MDP without considering communication. In this framework, each arm has a binary communication action. Choosing to communicate allows an arm to receive the local Q function parameters from a similar arm. Using these parameters, the receiving arm establishes a behavior policy to collect experience samples for updating its local Q function. This behavior policy may enable exploration and help escape false optima. In this way, the communication reward can be calculated as the difference in RMAB return with and without communication. We then adopt Q-learning for communication learning. Since RMABs optimize a centralized objective involving return from all the arms, the communication reward captures the joint influence of communication of all arms.
To solve this problem and learn the individual communication strategy of each arm, we propose to use a decomposed Q-network architecture for communication learning that factorizes the joint communication Q function into local communication Q function.

To demonstrate the effectiveness our method, we rigorously compare the sample complexity required to learn a near-optimal per-arm Q function when learning with and without communication. 
We show theoretically that there exists RMABs where communication from non-noisy to noisy arms can reduce the Q estimation errors exponentially with respect to the number of states. 
Interestingly, further studies reveal that communication from noisy to non-noisy arms can also be beneficial, provided that the behavior policy, on the receiver arms' MDP, achieves reasonable coverage over the state-action space and does not decorrelate too slowly from the initial states.

Empirically, we validate our method on synthetic problems, as well as the ARMMAN maternal healthcare~\cite{vermagroup} and the SIS epidemic intervention problem~\cite{yaesoubi2011generalized} built upon real-world data. Results show that our method significantly outperforms both the non-communicative learning approach and the approach with fixed communication strategies, achieving performance levels comparable to those learning in noise-free environments. Results are robust across problems, resources budgets, and noise levels. Furthermore, visualization of the learned communication strategies supports our theoretical findings that communication from noisy arms to non-noisy arms, and vice versa, can be helpful.


\textbf{Related works}.\ To the best of our knowledge, this is the first work that introduces communication into RMABs. This is perhaps surprising given that communication is an important topic that has been extensively studied in other multi-agent problems. In Dec-POMDPs~\cite{oliehoek2016concise,kwak2010teamwork,nair2004communications}, previous works explore decentralized communication learning~\cite{sukhbaatar2016learning,hoshen2017vain, jiang2018learning, singh2019learning,wang2019learning}, the emergence of natural language~\cite{foerster2016learning,das2017learning,lazaridou2017multi, mordatch2018emergence,zhao2024ai4sg,kang2020incorporating}, and implicit communication by coordination graphs~\cite{bohmer2020deep,wang2021context,kang2022non,yang2022self}, influence on each other~\cite{jaques2019social,wang2019influence}, and attention mechanisms~\cite{huang2020one,kurin2020my,dong2022low,dong2023symmetry}. Different from these works, we consider communication under noisy returns and in the RMAB setting where the central planner policy constrained by resource budgets determines the influence of communication. 

In the literature of cooperative multi-agent reinforcement learning~\cite{yu2022surprising,wen2022multi,kuba2021trust,wang2020roma,christianos2020shared,peng2021facmac,wang2020dop,jiang2019graph,wang2020rode,nagaraj2023multi}, value decomposition methods~\cite{sunehag2018value,rashid2018qmix,son2019qtran,li2021celebrating} have recently witnessed great success in addressing challenges such as partial observability and learning scalability. 
These methods decompose the joint Q-value conditioned on true states and joint actions into per agent local Q functions based on local action-observation history and local actions. 
In this work, we adopt a linear value decomposition framework~\cite{sunehag2018value} for the communication Q network to solve the unique problems of learning with the presence of noisy arms.


\section{Preliminaries}

\subsection{RMABs with Noisy Arms}

We study multi-action Restless Multi-Arm Bandits (RMABs) with $N$ arms. Each arm $i\in[N]$ is associated with a Markov Decision Process $\mathcal{M}_i=(\mathcal{S}_i, \mathcal{A}_i, \mathcal{C}_i, T_i, R_i,\beta, \bm z_i)$. Here, $\mathcal{S}_i$ is the state space that can be either discrete or continuous. The reward at a state is given by a function $R_i: \mathcal{S}_i \rightarrow\mathbb{R}$, and $\beta\in[0,1)$ is a discount factor for future rewards. We assume that there is a set of \textbf{noisy arms} $\mathcal{N}_\epsilon\subset [N]$, and the rewards of a noisy arm $i\in \mathcal{N}_\epsilon$ are unreliable and affected by noise $\epsilon_i(s_i)$ at some states $s_i\in\mathcal{S}_{\epsilon,i}\subset\mathcal{S}_{i}$. We consider systematic errors: $\tilde{R}_i(s_i) = R_i(s_i)+\epsilon_i(s_i)$, where $\epsilon_i(s_i)$ is sampled from a distribution $\mathcal{E}_i$ and fixed during training and testing. The identity of noisy arms $\mathcal{N}_\epsilon$, the affected states $\mathcal{S}_{\epsilon,i}$, the reward noise $\epsilon_i$, and its distribution $\mathcal{E}_i$ are not known a priori.

$\mathcal{A}_i$ is a finite set of discrete actions. Each action $a\in\mathcal{A}_i$ has a cost $\mathcal{C}_i(a)$. Following the standard bandit assumption that an arm can be "not pulled" at no cost, we set $\mathcal{C}_i(0)=0$. The probability of transitioning from one state to another given an action is specified by the function $T_i: \mathcal{S}_i\times \mathcal{A}_i \times \mathcal{S}_i \rightarrow[0,1]$. In line with previous work~\cite{zhao2023towards, elmachtoub2023balanced, elmachtoub2023estimate, zhao2022implicit}
we assume that $\mathcal{S}_i, \mathcal{A}_i$, and $\mathcal{C}_i$ are the same for all arms $i\in[N]$ and omit the subscript $i$ for simplicity. Each arm has a feature vector $\bm z_i\in\mathbb{R}^m$ that provides useful information about the transition dynamics of arm $i$. We define the arm with the most similar features with arm $i$ as $\nu_i=\arg\min_j\|z_j-z_i\|_2$.


We consider a constrained global resource setting where, at every timestep $t\in[H]$, where $H$ is the horizon, the total cost of actions taken is no greater than a given budget $B$. Under this constraint, the \emph{central planner} selects one action for each of the $N$ arms at each timestep. Formally, the central planner has a policy $\pi$ to maximize its expected discounted return $G(\pi) = \mathbb{E}_\pi[\sum_{t=1}^H \beta^t\sum_{i\in[N]}R_i(s_i^t)]$, where $s_i^t$ is the state of arm $i$ at time step $t$. $\pi$ is typically computed by solving the constrained Bellman equation~\cite{hawkins2003langrangian,killian2021beyond,killian2022restless}:
\begin{align}
\label{eq:original_opti_problem}
&J(\boldsymbol{s})=\sum_{i=1}^N R_i\left(s_i\right)+\beta \, \max _{\boldsymbol{a}}\mathbb{E}_{\boldsymbol{s}^{\prime}}\left[J\left(\boldsymbol{s}^{\prime}\right) \mid \boldsymbol{s},\boldsymbol{a}\right] \nonumber\\
&\text { s.t. } \sum_{i=1}^N \mathcal{C}(a_i)\leq B .
\end{align}

\subsection{Reinforcement Learning in RMABs}

Deep reinforcement learning has been proven to be efficient in solving the RMAB problem~\cite{avrachenkov2022whittle,killian2022restless,zhao2024towards}. To learn arm policies, we take the Lagrangian relaxation of the budget constraint~\cite{hawkins2003langrangian}, defining
\begin{align}
    J(s,\lambda) = \frac{\lambda B}{1-\beta}+\sum_{i=1}^N \max_{j\in[A]}\left\{Q_i(s_i,j)\right\},
\end{align}
where $Q_i$ is arm $i$'s Q function estimating its expected discount return of an action at a state. We use deep Q-learning
and parameterize the action-value function $Q_i(s_i, a_i; \theta_i)$ of each arm $i$ with $\theta_i$, which is updated by the TD-loss \cite{sutton2018reinforcement,killian2021beyond}:
\begin{align}\label{equ:q_td_loss}
\mathcal{L}_\texttt{TD}(\theta_i) =& \mathbb{E}_{(s_i, a_i, s_i')\sim\mathcal{D}}\Big[\Big(R_i(s_i)-\lambda \mathcal{C}(a_i)\nonumber\\ 
&+ \beta \max_a Q_i(s_i',a;\theta_i)-Q_i(s_i,a_i;\tilde{\theta}_i)\Big)^2\Big],
\end{align}
where $Q_i(\cdot;\tilde\theta_i)$ is a target network whose parameters are periodically copied from $Q_i$. The expectation means that the samples come from a replay buffer $\mathcal{D}$.


To \textit{infer} a central planner policy that adheres to the budget constraint, previous research~\cite{killian2022restless} constructs a knapsack problem using $Q_i(s_i, a_i)$ as values and $\mathcal{C}(a_i)$ as weights, within the constraints defined in the optimization problem \ref{eq:original_opti_problem}. The solution to this knapsack problem identifies the joint action that maximizes the sum of the learned $Q_i$ values while conforming to the budget constraint. In this case, we denote the planner policy by $\pi(\bm\theta)$ to highlight its dependency on the joint parameters of individual Q-functions $\bm\theta=(\theta_1,\cdots,\theta_N)$.


\section{RMAB Learning with Communication}

In this section, we introduce our method of learning communication to share useful knowledge among arms. We first define the communication problem as a Multi-Agent MDP~\cite{boutilier1996planning} and address the following problems: (1) What are 
the messages? (Sec.~\ref{sec:method:comm_problem}) How do they affect local Q learning? (Sec.~\ref{sec:method:comm_dynamics}) (2) How to generate reward signals that measure the influence of communication? (Sec.~\ref{sec:method:comm_reward}) (3) How to use the communication reward to learn communication strategies? (Sec.~\ref{sec:method:comm_decomposed}) (4) Why and under what conditions can communication help learning? (Sec.~\ref{sec:method:comm_dynamics} and~\ref{sec:theory})


\subsection{Define the Communication Learning Problem}\label{sec:method:comm_problem}

Akin to traditional deep RMAB learning algorithms~\cite{fu2019towards,biswas2021learn,killian2022restless,avrachenkov2022whittle}, our approach involves each arm learning its local Q networks $Q_i(s_i, a_i; \theta_i)$ to predict the expected return of taking action $a_i$ in state $s_i$. The central planner then selects centralized actions based on the local Q values.

When the data source is noisy, the local Q functions may be inaccurate, leading to the central planner choosing suboptimal actions. 
Figure 3 in Appendix C.3 shows such an example. To address this issue, we introduce a mechanism that allows each arm to learn a communication policy that diminishes the impact of noisy data and enhances the overall effectiveness of the central planner.

We define the communication learning problem within the framework of a Multi-Agent MDP~\cite{boutilier1996planning}.
\begin{definition}[The Communication MDP]
    The communication MDP in the RMAB setting is modelled by a multi-agent MDP $\mathcal{M}^\texttt{c}=(\mathcal{S}^\texttt{c}, I, \mathcal{A}^\texttt{c}, T^\texttt{c}, \comr)$. $I$ is the set of players, which in our case is the set of arms. $\mathcal{S}^\texttt{c}=\Theta$ is the state space, where $\Theta$ is the joint parameter space of Q networks. $\bm\theta=(\theta_1,\cdots,\theta_N)\in \mathcal{S}^\texttt{c}$ corresponds to joint Q function parameters. $\mathcal{A}^\texttt{c}=\{0,1\}^I$ is the action space. $\coma{i}=1$ means arm $i$ decides to ask Q network parameters from the arm with the most similar features, $\nu_i$, as the message. If $\coma{i}=0$, arm $i$ will not receive a message. The joint communication action $\jcoma$ leads to a transition in $\bm\theta$ (Eq.\ref{equ:coms_t}) and a reward $\comr$ (Eq.~\ref{equ:comm_r}).
\end{definition}
The definition of this Communication MDP is associated with the underlying RMAB in the sense that the Communication MDP solution affects the action selection of the central planner. After every $K$ learning epochs of $Q_i$ in the MDP $\mathcal{M}_i$, parameters of arms' Q functions are used as the state for this Communication MDP $\mathcal{M}^\texttt{c}$. We assume that each arm $i$ in the system has the chance to communicate with the arm with the most similar features, denoted by arm $\senderfor{i}$. This results in $N$ possible communication channels (i.e. arm $i$ will receive a message if $\coma{i}=1$). In this way, the set of players in $\mathcal{M}^\texttt{c}$ is the set of arms $[N]$. In addition to selecting actions in the arm MDP $\mathcal{M}_i$, each arm also selects actions in this Communication MDP $\mathcal{M}^\texttt{c}$. These \emph{communication actions} determine the communication channel activation pattern, which remains frozen for the next $K$ learning epochs. We will discuss the details in the following subsections.

An alternative design option is to adopt a denser communication scheme, \ie, an arm can receive messages from multiple other arms. As formally detailed in Proposition~\ref{prop:sparse_com} in Sec.~\ref{sec:theory}, under certain assumptions regarding concentratability~\cite{munos2003error,chen2019information}, our sparse communication scheme is shown to effectively learn near-optimal value functions with the same errors as the dense scheme. Therefore, we favor the sparse scheme due to its lower demands on communication bandwidth.


\subsection{Communication Actions and Transition Dynamics}\label{sec:method:comm_dynamics}
We now describe how the communication action $\jcoma$ induces a state transition in the space $\mathcal{S}^\texttt{c}$. Specifically, when $\coma{i}$ is 1, arm $i$ receives a message from arm $\senderfor{i}$, which is the parameters of the local Q network of arm $\senderfor{i}$, $\theta_{\nu_i}$. Using this message, arm $i$ formulates a behavior policy $\pi_{{\senderfor{i}}}$ as follows:
\begin{align}
    \pi_{{\senderfor{i}}}(\cdot|s_i)=\texttt{SoftMax}\left(Q_{\senderfor{i}}(s_i,\cdot)\right).
\end{align}
Using behavior policy $\pi_{{\senderfor{i}}}$, arm $i$ navigates in its arm MDP $\mathcal{M}_i$, collecting transition samples in a replay buffer $\mathcal{D}_{\pi_{\senderfor{i}}}$. This buffer is then used to update the local Q-network. Therefore, the TD-loss for arm $i$ with the communication channel $i$ activated is given by:
\begin{align}
\mathcal{L}^\texttt{c}_\texttt{TD}(\theta_i|\coma{i}\shorte 1) = & \mathbb{E}_{(s_i, a_i, s_i')\sim\mathcal{D}_{\pi_{\senderfor{i}}}}\Big[\Big(R_i(s_i)-\lambda c_{a_i} \nonumber\\
&+ \beta \max_a Q_i(s_i',a;\theta_i)-Q_i(s_i,a_i;\tilde{\theta}_i)\Big)^2\Big].\nonumber
\end{align}

Alternatively, if arm $i$ opts not to receive communication, it engages in the $\epsilon$-greedy strategy to explore the MDP $\mathcal{M}_i$ and generate the replay buffer. Therefore, 
\begin{align}
    \mathcal{L}^\texttt{c}_\texttt{TD}(\theta_i|\coma{i}\shorte 0)=\mathcal{L}_\texttt{TD}(\theta_i).
\end{align}
As such, the updates to local Q parameters are dependent on the communication action:
\begin{align}
    \theta_i'(\coma{i}) = \theta_i - \alpha\nabla_{\theta_i}\mathcal{L}^\texttt{c}_\texttt{TD}(\theta_i|\coma{i}),\label{equ:coms_t}
\end{align}
where $\alpha$ is the learning rate. Eq.~\ref{equ:coms_t} specifies the state transition dynamics of the Communication MDP.


\subsection{Communication Reward}\label{sec:method:comm_reward}

Lemma~\ref{lemma:r2u_useful} and~\ref{lemma:r2u_useless} in Sec.~\ref{sec:theory} reveal that some seemingly reasonable fixed communication patterns, 
like those from a non-noisy arm to a noisy arm, 
can lead to an exponential higher sample complexity for learning a near-optimal Q function in the noisy arm. In fact, as shown in Proposition~\ref{prop:utor} of Sec.~\ref{sec:theory}, the influence of communication is determined by the state-action space coverage and the mixing time of the behavior policy $\pi_{\nu_i}$ on the receiver arms.


These findings highlight the intricate and dynamic nature of the communication learning problem in RMAB. In practice, to adaptively decide the optimal communication action, we propose that each arm learns a communication Q function $\comQ{i}$. This approach encounters obstacles in the RMAB setting, where clear, separate reward signals for each $\comQ{i}$ are absent. The central planner policy instead bases action decisions on the updated Q-functions of all arms and the budget. The collaborative contribution of the joint communication action $\jcoma$ to the learning objective is:
\begin{align}\label{equ:comm_r}
    \comr(\bm\theta,\jcoma) = G_T\left(\pi(\bm\theta'(\jcoma))\right) - G_T\left(\pi(\bm\theta'(\bm 0))\right).
\end{align}
Here, $G_T(\pi) = \frac{1}{T}\sum_{k=1}^T[\sum_{t=1}^H \beta^t\sum_{i\in[N]}R_i(s_i^{k,t})]$ is the empirical return under policy $\pi$ averaged over $T$ episodes,
where $s_i^{k,t}$ is arm $i$ state at timestep $t$ of episode $k$. $\pi(\bm\theta'(\jcoma))$ and $\pi(\bm\theta'(\bm 0))$ are the planner policies with the arm Q-functions updated with and without messages, respectively. The communication reward is defined as the difference in the planner's return in these two cases.

\subsection{Decomposed Communication Q Function}\label{sec:method:comm_decomposed}
With the definition of the communication reward (Eq.~\ref{equ:comm_r}), the simplest strategy of communication learning is to estimate a \emph{joint} communication action-value function $\jcomQ(\bm\theta, \jcoma)$. However, the question is that there are $N$ communication channels, and even though our communication action is binary, there are $O(2^{N})$ possible joint actions $\jcoma$. With this exponential number of actions, it is challenging to conduct deep Q-learning, which requires a maximization operation over the whole action space.

We propose to solve this problem by value decomposition~\cite{sunehag2018value}. Specifically, we represent the global communication Q-function as a summation of the channel-level, individual communication Q-functions:
\begin{align}\label{equ:decomp}
    \jcomQ(\bm\theta, \jcoma;\bm\xi) = \sum\nolimits_{i\in[N]}\comQ{i}(\bm\theta, \coma{i};\xi_{i}).
\end{align}
Here, $\xi_{i}$ is the parameters of the channel-level communication Q-function, and $\bm\xi$ is the joint parameter consisting of all $\xi_{i}$. Such a decomposition satisfies an important property:
\begin{align}
    \arg\max_{\jcoma} \jcomQ = \left(\arg\max_{\coma{1}}\comQ{1},\cdots,\arg\max_{\coma{N}}\comQ{N}\right),
\end{align}
which is because Eq.~\ref{equ:decomp} allows $\partial\jcomQ/\partial \comQ{i}\ge 0$. This means that the joint action that maximizes the $\jcomQ$ can be obtained by independently selecting the local communication action that maximize $Q_i$. As a whole, network $\jcomQ(\cdot;\bm\xi)$ is a standard Q-network conditioned on states $\bm\theta$ and actions $\jcoma$, allowing it to be trained using joint rewards $\comr$ (Eq.~\ref{equ:comm_r}). On a finer level, it has a decomposable structure simplifying the $\arg\max$ calculation. Therefore, $\jcomQ(\cdot;\bm\xi)$ can be trained by the following TD-loss:
\begin{align}
&\mathcal{L}^\texttt{d}_\texttt{TD}(\bm\xi) = \mathbb{E}_{(\bm\theta,\jcoma,\bm\theta')\sim\mathcal{D}}\Big[\Big(\comr(\bm\theta,\jcoma)  \nonumber\\
& +\beta \sum\nolimits_{i\in[N]}\max_{a} \comQ{i}(\bm\theta', a; \tilde{\xi}_i)-\sum\nolimits_{i\in[N]}\comQ{i}(\bm\theta, \coma{i};\xi_{i})\Big)^2\Big].\nonumber
\end{align}
$\comQ{i}(\cdot; \tilde{\xi}_i)$ is a target network with parameters periodically copied from $\comQ{i}$, and $\mathcal{D}$ is a replay buffer. In this way, we implicitly achieve credit assignment, because we use a joint reward $\comr$ for training but get a set of Q functions for individual arms, which can recover the best joint action locally.


 \section{Theoretical Analysis}\label{sec:theory}
We begin by investigating the influence of communication and establishing the specific conditions under which communication can reduce the sample complexity for the receiver arm in learning a near-optimal Q function. We define the minimum state-action occupancy probability,
\begin{align}
    \mu_{\texttt{min}} := \min_{s,a}\mu_{\pi_{\nu_i}}(s,a),
\end{align}
as the lowest probability of any state-action pair under the stationary distribution $\mu_{\pi_{\nu_i}}$ induced by the behavior policy $\pi_{\nu_i}$ on the receiver MDP, where $\nu_i$ is the arm with the most similar feature to arm $i$. 
We further define the mixing time~\cite{hsu2015mixing,li2020sample} $t_\texttt{mix}:=$
\begin{align}
    \min\left\{t|\max_{s_0,a_0}\sup_{s,a} |P^t(s,a|s_0,a_0)-\mu_{\pi_{\nu_i}}(s,a)|\le\frac{1}{4}\right\},
\end{align}
where $P^t$ is the state-action distribution at timestep $t$ given the initial state and action. The mixing time $t_\texttt{mix}$ captures how fast the behavior policy decorrelates from the initial state. We now show the condition under which communication can be helpful.

\begin{restatable}[Useful Communication]{proposition}{usefulcom}\label{prop:utor}
For the Q-learning of arm $i$, when
\begin{align}
    \mu_{\texttt{min}}>\frac{2(1-\beta)^2}{|\mathcal{S}||\mathcal{A}|},\ \ \text{and}\ \ t_\texttt{mix}\le \frac{1}{\epsilon_e^2(1-\beta)^4},
\end{align}
with a probability at least $(1-\delta)^2$ for any $0<\delta<1$, compared to learning without communication, choosing to receive communication and using behavior policy $\pi_{\nu_i}$ leads to better sample complexity in the worst case for learning $\epsilon_e$-optimal Q function such that $\|Q_i(s,a)-Q_i^*(s,a)\|<\epsilon_e, \forall (s,a)$.
\end{restatable}

The Proof is detailed in Appendix C.2
Proposition~\ref{prop:utor} shows that communication can accelerate arm $i$'s Q-learning as long as $\pi_{\nu_i}$ visits every state-action pair with a small probability and is not too slowly disentangling from the initial state. This finding is solely dependent upon the characteristics of the behavior policy in the receiver MDP $\mathcal{M}_i$ and does not require the sender or the receiver to be non-noisy or noisy.

We then analyze the effectiveness of a sparse communication strategy. Specifically, we compare our setup where each arm $i$ receives messages from an arm $\nu_i$ against a dense communication scheme where arm $i$ receives messages from $d$ other arms. In the dense communication scheme, the behavior policy $\pi_{\bar{\nu}_i}$ is the average of policies from all senders. Specifically, we collect the same number of samples using each sender's policy and amalgamate them into a unified experience replay buffer.


The following assumption on concentratability asserts that both $\pi_{\nu_i}$ and $\pi_{\bar{\nu}_i}$ provide a reasonable coverage over the state-action space.
\begin{assumption}[Concentratability Coefficient]
There exists $C_{\nu_i}<\infty$, $C_{\bar{\nu}_i}<\infty$ s.t.
\begin{align}
    \frac{v(s,a)}{\mu_{\bar{\nu}_i}(s,a)}\le C_{\bar{\nu}_i} \text{ and } \frac{v(s,a)}{\mu_{\nu_i}(s,a)}\le C_{\nu_i} \text{, for } \forall(s,a)\in(\mathcal{S}, \mathcal{A}),\nonumber
\end{align}
for any admissible distribution $v$ that can be generated by following some policy for some timesteps in the MDP $\mathcal{M}_i$. Here $\mu_{\nu_i}$ and $\mu_{\bar{\nu}_i}$ are the stationary distribution of $\pi_{\nu_i}$ and $\pi_{\bar{\nu}_i}$ on $\mathcal{M}_i$, respectively.
\end{assumption}
Under this assumption, we can prove that the sparse communication scheme can also learn a near-optimal Q function as under the dense scheme. The key observation that underpins the proof is $\frac{1}{d}\mu_{\nu_i}(s,a)\le \mu_{\bar{\nu}_i}(s,a)$, where $d<N$ is the maximum number of other arms that can send messages to arm $i$. The convergence of the communication learning under the dense scheme implies that $C_{\bar{\nu}_i}$ is finite, which further indicates the finiteness of $C_{\nu_i}$ and thus the convergence of the communication learning of the sparse scheme.
\begin{restatable}[Sparse Communication]{proposition}{sparsecom}\label{prop:sparse_com}
For a dense communication strategy, there exists a sparse communication strategy which approximates the optimization step in Equation~\eqref{equ:coms_t} faithfully. Additionally, both these two strategies can learn near-optimal value functions, and the sparse communication scheme requires at most $O(d)$ times more samples to achieve this.
\end{restatable}

Detailed proof is deferred to Appendix C.1. 
We next show that a communication channel from a non-noisy arm to a noisy arm can be either helpful (Lemma~\ref{lemma:r2u_useful}) or detrimental (Lemma~\ref{lemma:r2u_useless}), depending on the specific structures of the sender and receiver MDPs.


\begin{restatable}{lemma}{ruuseful}\label{lemma:r2u_useful}
    There exists a family of RMABs where, without communication from non-noisy arms to noisy arms, the error in Q values of a noisy arm could be $\max_{s,a}|Q_i^*(s,a)-Q_i(s,a)|=O\left((\frac{1}{\beta})^{|\mathcal{S}|}\right)$ even after $O\left((\frac{1}{\epsilon})^{|\mathcal{S}|}\right)$ epochs. $\epsilon\in(0,1)$ is the exploration rate in $\epsilon$-greedy exploration scheme. By contrast, this Q error can be reduced to 0 with communication from non-noisy arms to noisy arms.
\end{restatable}

\begin{restatable}{lemma}{ruuseless}\label{lemma:r2u_useless}
    There exists a family of RMABs where, with communication from a non-noisy arm to a noisy arm, we need $O\left((\frac{1}{\epsilon})^{|\mathcal{S}|}\right)$ times more samples to learn a near-optimal Q function in the noisy arm than without communication.
\end{restatable}

The proofs of Lemma~\ref{lemma:r2u_useful} and~\ref{lemma:r2u_useless} can be found in Appendix C.3. 
The observation that the errors in the Q function and the sample complexity of Q learning can increase exponentially with the number of states in the worst case scenarios highlights the need to move away from predefined communication channels. Instead, communication Q-learning provides a more adaptable and effective approach.


\section{Experiments}
\begin{figure*}[t]
\includegraphics[width=\textwidth]{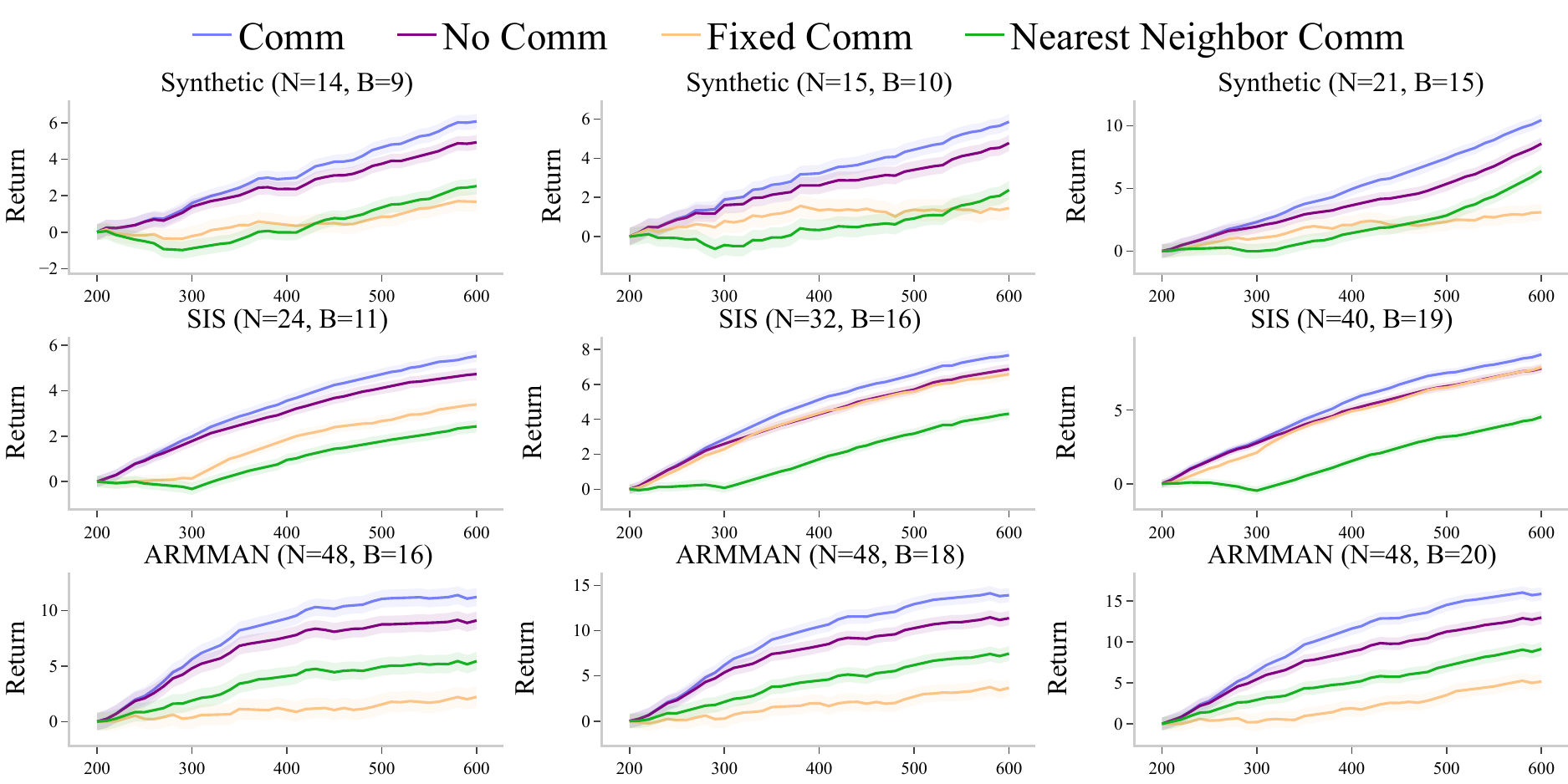}
        \caption{Performance (interquartile mean \cite{agarwal2021deep} and standard error of return over 200 random seeds) of our method, baselines, and ablations in three environments with different numbers of arms $N$ and resource budgets $B$. Communication learning starts at epoch 200. We present results on \textbf{two additional baseline} in Tables~\ref{table:noise_level_ablation} and \ref{table:num_noisy_arms_ablation} (see Appendix B for more ablations).}
        \label{fig:main_result_figure}
\end{figure*}

We present experimental evaluations and ablation studies to investigate: (1) The learning performance of our communication method compared to baselines; (2) The characteristics of learnt communication strategies; (3) The robustness of our method in terms of resource budgets, the number of arms, noise types and levels. We discuss additional experimental details and hyperparameters in Appendix B.


We evaluate our methods on the following domains. In all domains, 80\% of the arms are noisy, and noisy arms observe reward drawn from a perturbed reward function $\tilde{R}(s) = R(s)+\epsilon(s)$, where $\epsilon\sim\mathcal{N}(0,1)$. In ablation studies (Tables~\ref{table:num_noisy_arms_ablation} and 3), we test various numbers of noisy arms and noise type. 

\textbf{Synthetic}: We consider a synthetic RMAB with continuous states. For state $s_{i}\in[0, 1]$ and action $a\in\{0,1\}$ of arm $i$, the next state $s_i'$ is determined by: $s_{i}' =  clip\left(s_i + \mathcal{N}(\mu_{ia}, \sigma_{ia}), 0, 1\right)$. Here transition parameters are sampled uniformly, specifically $\mu_{i1} = -\mu_{i0} \in [-0.2, 0.2]$, $\sigma_{i1} = \sigma_{i0} \in [0.1, 0.2]$.

\textbf{SIS\ Epidemic\ Model}~\cite{yaesoubi2011generalized,killian2022restless, zhao2024towards}: Each arm $p$ represents a subpopulation in a geographic region, with the number of uninfected individuals being the state $s_p$. We implement a budget-constrained intervention framework with three actions: \(a_0\) represents no intervention, \(a_1\) requires weekly COVID testing, and \(a_2\) involves distributing face masks. Actions change the disease transmission parameters in a targeted manner, reflecting their real-world public-health impacts. The reward noise is due to variability in data collection, which may be caused by poor COVID testing equipment \cite{embrett2022barriers,schmitt2020covid} and influence of authorities or impacted populations (keeping the COVID numbers low) \cite{paulus2023reinforcing}. 

\textbf{ARMMAN}: The dataset is collected by ARMMAN~\citep{armman-mhealth}, an NGO in India aimed at improving health awareness for a million expectant and new mothers via automated  voice messaging programs. Each arm represents a cluster of mothers in a geographical location. Actions are sending a health worker to visit mothers in a location to improve engagement. States, continuous valued in [0,1], represent the average engagement of mothers in a location ~\citep{danassis2023limited}. The transition dynamics is established empirically from the data (data usage and consent are in Appendix A.2). 
Systematic reward errors are present due to difference in infrastructure such as cellular network and electricity \cite{kumar2022survey} (e.g. a geographic location may have poor cellular signals or frequent power outage, causing a systematic underestimation of mothers' engagement).  

\begin{table*}[htb]
\centering
\scalebox{.8}{
    \begin{tabular}{l@{}p{0.1mm}c ccc@{} c ccc@{} c ccc@{}}
    \toprule
    \multirow{1}{*}{} & \multirow{3}{*}{\ \ } & \multicolumn{3}{c}{\textbf{Synthetic (N=21, b=15)}} 
     & & \multicolumn{3}{c}{\textbf{SIS (N=32, B=16)}}
     & & \multicolumn{3}{c}{\textbf{ARMMAN (N=48, B=20)}}\\
    \cmidrule{3-5} \cmidrule{7-9} \cmidrule{11-13}
\textbf{Method}& & 100\% & 110\%  & 120\% & & 100\%  & 110\% & 120\% & & 100\%  & 110\% & 120\%\\
    \midrule
    \sname{No\ Noise} & & 100\%  & 100\%  & 100\% & & 100\%  & 100\%  & 100\% & & 100\% & 100\% & 100\%\\ 
    \sname{Comm} & & 75\%  & 53\%  & 38\% & & 63\%  & 55\%  & 51\% & & 77\% & 72\% & 68\%\\  
    \sname{No\ Comm} & & 62\%  & 39\%  & 24\% & & 55\%  & 45\%  & 42\% & & 63\% & 59\% & 52\%\\  
    \sname{Fixed/ Comm} & & 21\%  & 6\%  & 12\% & & 56\%  & 42\%  & 46\% & & 24\% & 15\% & 27\%\\  
    \sname{Nearest\ Neighbor\ Comm} & & 44\%  & 20\%  & 3\% & & 32\%  & 22\%  & 23\% & & 46\% & 43\% & 31\%\\   
    \sname{Random\ Comm} & & -24\%  & -45\%  & -58\% & & 22\%  & 13\%  & 9\% & & -22\% & -26\% & -35\%\\     
    \bottomrule
    \end{tabular}
}
\caption{Average returns (100 seeds) with increased noise, normalized with respect to \sname{No\ Noise} and \sname{Comm} (ours) before communication (see Remark 7 in Appendix B). We test Gaussian Noise $\mathcal{N}(0,\sigma)$. $x\%$ noise level means $\sigma=x/100$.}
\label{table:noise_level_ablation}
\end{table*}

\begin{table*}[htb]
\centering
\scalebox{.8}{
    \begin{tabular}{l@{}p{0.1mm}c ccc@{} c ccc@{} c ccc@{}}
    \toprule
    \multirow{1}{*}{} & \multirow{3}{*}{\ \ } & \multicolumn{3}{c}{\textbf{Synthetic (N=21, B=15)}} 
     & & \multicolumn{3}{c}{\textbf{SIS (N=32, B=16)}}
     & & \multicolumn{3}{c}{\textbf{ARMMAN (N=48, B=20)}}\\
    \cmidrule{3-5} \cmidrule{7-9} \cmidrule{11-13}
\textbf{Method}& & 15 & 16  & 17 & & 22  & 24 & 26 & & 34  & 36 & 38\\
    \midrule
    \sname{No\ Noise} & & 100\%  & 100\%  & 100\% & & 100\%  & 100\%  & 100\% & & 100\% & 100\% & 100\%\\ 
    \sname{Comm} & & 78\%  & 78\%  & 75\% & & 69\%  & 71\%  & 63\% & & 82\% & 80\% & 77\%\\  
    \sname{No\ Comm} & & 67\%  & 66\%  & 62\% & & 61\%  & 61\%  & 55\% & & 68\% & 59\% & 63\%\\  
    \sname{Fixed/ Comm} & & 63\%  & 45\%  & 21\% & & 65\%  & 62\%  & 56\% & & 43\% & 49\% & 24\%\\  
    \sname{Nearest\ Neighbor\ Comm} & & 42\%  & 47\%  & 44\% & & 11\%  & 29\%  & 32\% & & 47\% & 37\% & 46\%\\   
    \sname{Random\ Comm} & & -46\%  & -32\%  & -24\% & & -8\%  & 17\%  & 22\% & & -60\% & -30\% & -22\%\\   
   \bottomrule
    \end{tabular}
}
\caption{Performance (return averaged across 100 random seeds) with numbers of noisy arms, normalized with respect to \sname{No\ Noise} and \sname{Comm} (ours) before communication starts (see Remark 7 in Appendix B).}
\label{table:num_noisy_arms_ablation}
\end{table*}

\textbf{Learning Performance}.\ \ Figure~\ref{fig:main_result_figure} shows the experimental results in \sname{Synthetic}, \sname{SIS}, and \sname{ARMMAN} where we compare our method \sname{Comm\ (Ours)} against various baselines and ablations. \sname{No\ Comm} is an ablation where the RMAB with noisy arms is trained without communication. \sname{Fixed\ Comm} involves a preset, fixed communication pattern where at each timestep, each noisy arm receives information from a similar ($\ell_2$ distance of features), noise-free arm. \sname{Fixed\ Comm} uses oracle information about which arms are noisy, which is not known by other methods. \sname{Nearest\ Neighbor\ Comm} is a baseline strategy where an arm always receives communication from the nearest neighbor in feature space, regardless of whether the sender/receiver is a noisy arm. We include \textbf{two additional baselines} in Tables~\ref{table:noise_level_ablation} and \ref{table:num_noisy_arms_ablation}. 

Our method significantly outperforms \sname{No\ Comm}, \sname{Fixed\ Comm}, and \sname{Nearest\ neighbor\ Comm}. Notably, the \sname{Fixed\ Comm} strategy, despite the oracle information, often performs worse than \sname{No\ Comm}. This observation aligns with our theoretical analysis in Proposition~\ref{prop:utor} showing that communication channels from non-noisy to noisy arms are not always helpful. 
Our method consistently outperforms across various number of arms, budgets, and domains. For instance, in the Synthetic environment with N=21 and B=15, our method achieves a return of approximately 10 at epoch 600, compared to about 8 for the no communication (No Comm) baseline, representing an 20 percentage point improvement. Similarly, in the ARMMAN environment with N=48 and B=20, our method reaches a return of 15, outperforming the No Comm baseline's return of approximately 12.5, which is a 20 percentage point increase. 
These results demonstrate the effectiveness of our communication algorithm in handling systematic errors and enhancing learning in noisy RMAB environments.

\begin{figure}[h!]
\centering
    \includegraphics[width=0.45\textwidth]{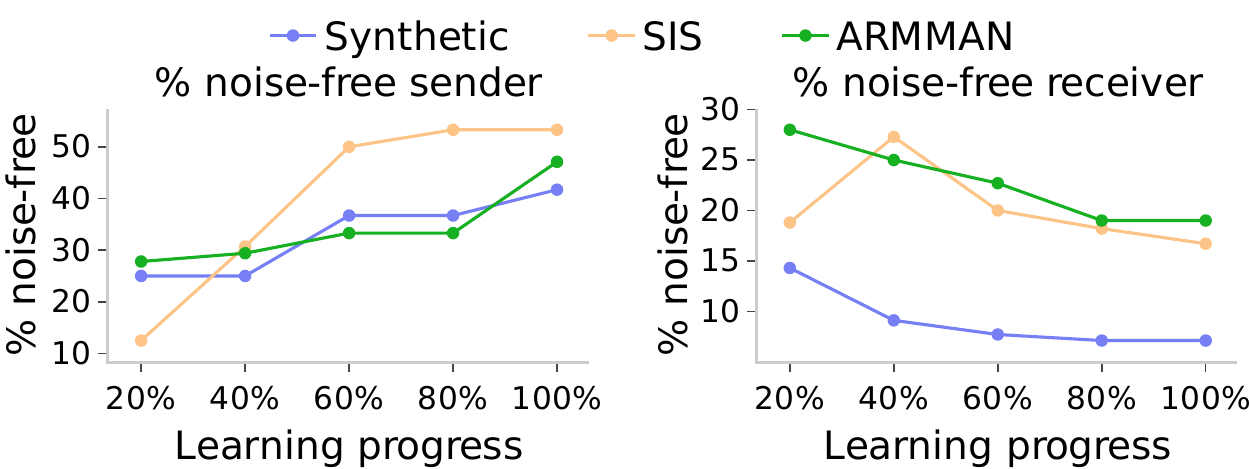}
    \caption{The change in the proportion of noise-free senders and receiver within all activated channels during the learning process. Results are for Synthetic (N=21, B=15), SIS (N=32, B=16), and ARMMAN (N=48,  B=16).}
    \label{fig:comm_pattern_figure}
\end{figure}

\textbf{Communication Pattern.} \ \ 
Figure~\ref{fig:comm_pattern_figure} illustrates how our learnt communication strategy evolves. We make several observations. First, as the number of learning epochs increase, we observe more information are sent from non-noisy arms to noisy arms. During the early stages of learning, when arms are primarily exploring, communication from noisy arms to non-noisy arms can be beneficial, as shown in Proposition~\ref{prop:utor}. As the learning progresses, the information from noisy sources may become less helpful as the arms have explored sufficiently. As a result, the proportion of communication from non-noisy arms increases. This observation is consistent across the three experimental environments. 

Second, notice the percentage of noise-free receivers do not drop to zero and noise-free senders are not at 100\% even at the end of training, indicating value of communication from noisy arms. Finally, the differing rates of noise-free arms as senders and receivers across domains indicate that our communication approach adapts to the domain.

\textbf{Robustness.}\ \ Tables~\ref{table:noise_level_ablation} and ~\ref{table:num_noisy_arms_ablation} show the performance of different methods under various settings (see Appendix B for more ablations). In \sname{No\ Noise}, all arms receive accurate, true reward signals without any noise. Arms learn individually on their own data without any communication. In \sname{Random\ Comm}, each arm receives communication from a randomly chosen arm. The results demonstrate that our method is robust under increasing noise levels, numbers of noisy arms, and different noise distributions. In contrast, the performance of baselines \sname{No\ Comm} and \sname{Fixed\ Comm} changes significantly when the environment changes. In particular, on Synthetic, when the number of noisy arms increases from 15 to 17 (Table~\ref{table:num_noisy_arms_ablation}), our \sname{Comm} method has \textbf{only 3\%} drop in performance, 
while \sname{Fixed\ Comm} has 42 percentage points drop in performance, despite that \sname{Fixed\ Comm} uses ground truth information about which arms are noisy that is not known to our \sname{Comm} method. Similar trends can be seen on SIS and ARMMAN. These results demonstrate that, compared to a fixed communication protocol with access to oracle information, our approach for communication learning 
provides better robustness. Additionally, the performance of \sname{No\ Comm} supports the argument that conventional RL methods struggle in noisy RMABs.


\section{Closing Remarks}

The study of RMABs in environments with noisy and unreliable data sources underscores the need for innovative solutions in RL methods. 
We introduce a novel communication learning algorithm to alleviate the influence of noisy data so that the central planner can make better coordination decisions. 
Theoretically, we prove that communication can help reduce errors in value function approximation if the behavioral policy's state-action occupancy probability is lower bounded by a suitable constant. Empirically, we show that our approach significantly improves RMAB 
performance. 


\section*{Acknowledgments}
The work is supported by Harvard HDSI funding

\bibliography{aaai25}

\appendix


\section{Ethics and Real-world ARMMAN Dataset}
\label{sec:appendix_armman_results}

We test our method in a purely simulated environment; any potential consideration of real-world deployment would require comprehensive field testing. All work in simulation was conducted in collaboration with ARMMAN's team of ethics, and any future consideration of deployment would require a real-world evaluation together with the domain partner, which may reveal further challenges to be addressed. 


\subsection{Secondary Analysis}
Our experiment falls into the category of secondary analysis of the data shared by ARMMAN. 
This paper does not involve the deployment of the proposed algorithm or any other baselines to the service call program. As noted earlier, the experiments are secondary analysis with approval from the ARMMAN ethics board.

\subsection{Consent and Data Usage}
\label{sec:appendix_consent_data_usage}
Consent is obtained from every beneficiary enrolling in the NGO's mobile health program. The data collected through the program is owned by the NGO and only the NGO is allowed to
share data. In our experiments, we use anonymized call listenership logs to calculate empirical transition probabilities. No personally
identifiable information (PII) is available to us.
The data exchange and usage were regulated by clearly
defined exchange protocols including anonymization, read-access
only to researchers, restricted use of the data for research purposes
only, and approval by ARMMAN’s ethics review committee.


\section{Additional Experiment Details and Results}\label{appx:exp_details}
\begin{table*}[htb!]
\centering
\scalebox{.99}{
    \begin{tabular}{l@{}p{0.1mm}c ccc@{} c ccc@{} c ccc@{}}
    \toprule
    \multirow{1}{*}{} & \multirow{3}{*}{\ \ } & \multicolumn{3}{c}{\textbf{Synthetic (N=21, B=15)}} 
     & & \multicolumn{3}{c}{\textbf{SIS (N=32, B=16)}}
     & & \multicolumn{3}{c}{\textbf{ARMMAN (N=48, B=20)}}\\
    \cmidrule{3-5} \cmidrule{7-9} \cmidrule{11-13}
\textbf{Method}& & \text{Gaussian} & \text{Uniform}  & \text{Mixture} & & \text{Gaussian}  & \text{Uniform} & \text{Mixture} & & \text{Gaussian}  & \text{Uniform} & \text{Mixture}\\
    \midrule
    \sname{No\ Noise} & & 100\%  & 100\%  & 100\% & & 100\%  & 100\%  & 100\% & & 100\% & 100\% & 100\%\\ 
    \sname{Comm} & & 75\%  & 97\%  & 79\% & & 63\%  & 73\%  & 56\% & & 77\% & 100\% & 80\%\\  
    \sname{No\ Comm} & & 62\%  & 92\%  & 71\% & & 55\%  & 66\%  & 47\% & & 63\% & 93\% & 75\%\\  
    \sname{Fixed/ Comm} & & 21\%  & 15\%  & 16\% & & 56\%  & 29\%  & 27\% & & 24\% & 37\% & 35\%\\  
    \sname{NN\ Comm} & & 44\%  & 89\%  & 62\% & & 32\%  & 48\%  & 20\% & & 46\% & 89\% & 72\%\\   
    \sname{Random\ Comm} & & -24\%  & 21\%  & 1\% & & 22\%  & 24\%  & -1\% & & -22\% & 35\% & 28\%\\   
  \bottomrule
    \end{tabular}
}
\caption{Performance (return averaged across 100 random seeds) with various noise distributions, normalized with respect to \sname{No\ Noise} and \sname{Comm} (ours) before communication starts. Gaussian noise is $\mathcal{N}(0,1)$, Uniform noise is $U(-0.5,0.5)$, and Mixture noise has 50\% chance of being the Gaussian and 50\% chance of being the Uniform. We abbreviate 
\sname{Nearest\ Neighbor\ Comm} as \sname{NN\ Comm}.}
\label{table:noise_shape_ablation}
\end{table*}

Our experimental settings include a synthetic setting, an epidemic modeling scenario, and a maternal healthcare intervention task, each designed to test the efficacy of our model under different conditions. In all experimental settings considered, 
to generate features, we randomly generate projection matrices and then project parameters that describe the ground truth transition
dynamics (unknown to the our learning algorithms) into features. The features have the same dimension as the number of parameters that describes the transition dynamics. In Table~\ref{table:noise_shape_ablation}, we present results on noise drawn from different distributions, including  Gaussian distribution, Uniform distribution, and Mixture distribution. 

In Table~\ref{table:hyperparameter}, we present \textbf{hyperparameters} used. We use categorical DQN to learn per-arm Q-values. The Q-network has tanh activation for all hidden layers and identity activation for output layers. 

\begin{table}[htb!]
    \centering
\scalebox{.99}{
    \begin{tabular}{@{}lc@{}}
        \toprule
        hyperparameter & value \\ 
        \midrule
        Number of atoms in Categorical DQN & 51\\
        Q-network learning rate & 5e-04\\
        Number of steps per epoch & 20\\
        DQN initial exploration epsilon & 0.3 \\
        DQN exploration epsilon decay rate & 0.999 \\
        Number of hidden layers in Q-network & 2 \\
        Number of neurons per hidden layer & 16 \\
        Target-Q updating frequency in epochs & 10 \\
        DQN buffer size in samples per arm & 12000 \\
        VDN-network learning rate & 5e-3\\
        \bottomrule
    \end{tabular}
    }
    \caption{Hyperparameter values.}
    \label{table:hyperparameter}
\end{table}

Our method is robust to both biased and unbiased reward errors. In Synthetic, with biased noise $\mathcal{N}(z,1)$ across varying $z$, our method consistently outperforms baselines, as shown in Table~\ref{table:biased_noise_ablation}.

\begin{table*}[htb!]
\centering
\scalebox{.99}{
    \begin{tabular}{l@{}p{0.1mm}c ccc@{}}
    \toprule
    \multirow{1}{*}{} & \multirow{3}{*}{\ \ } & \multicolumn{3}{c}{\textbf{Synthetic (N=21, B=15)}} \\
    \cmidrule{3-5} 
\textbf{Method}& & $\mathcal{N}(0,1)$ & $\mathcal{N}(0.1,1)$  & $\mathcal{N}(0.2,1)$ \\
    \midrule
    \sname{Comm} & & 75\%  & 64\%  & 56\% \\  
    \sname{No\ Comm} & & 62\%  & 50\%  & 43\% \\  
    \sname{Fixed/ Comm} & & 21\%  & 29\%  & 29\% \\  
    \sname{NN\ Comm} & & 44\%  & 33\%  & 24\% \\   
  \bottomrule
    \end{tabular}
}
\caption{Performance (return averaged across 100 random seeds) with biased noise $\mathcal{N}(z,1)$ across varying $z$, normalized with respect to \sname{No\ Noise} and \sname{Comm} (ours) before communication starts. We abbreviate 
\sname{Nearest\ Neighbor\ Comm} as \sname{NN\ Comm} .}
\label{table:biased_noise_ablation}
\end{table*}

\begin{remark}\label{remark:normalization}
Notice that in Tables 1, 2, 3, and 5, we normalize the returns of all algorithms under consideration by
$$
\frac{\text{return} - \text{return(\sname{Comm} at epoch 200)}}{\text{return(\sname{No\ Noise} at epoch 200}) - \text{return(\sname{Comm} at epoch 200})}.
$$
Communication learning starts at epoch 200.
\end{remark}


\section{Proofs}\label{appx:proof}

\subsection{Proposition~\ref{prop:sparse_com}}\label{appx:proof:sparse_com}
\sparsecom*

\begin{proof}
Suppose we consider dense communication, where the arm $i$ communicates with $d$ similar arms $\senderfor{i,1},\dots,\senderfor{i,d}$. In this case, the arm collects data from each of the policies $\mathcal{D}_{\pi_{\senderfor{i,1}}},\dots, \mathcal{D}_{\pi_{\senderfor{i,d}}}$. We call this mixture distribution $\bar{\mathcal{D}}_{i}$. Now, consider the TD-loss for this dense communication: 
\begin{align}
&\mathcal{L}^\texttt{c,dense}_\texttt{TD}(\theta_i|\coma{i}\shorte 1) = \mathbb{E}_{(s_i, a_i, s_i')\sim\bar{\mathcal{D}}_{i}}\Big[\Big(R(s_i)-\lambda c_{a_i} \nonumber\\
&+ \beta \max_a Q_i(s_i',a;\theta_i)-Q_i(s_i,a_i;\tilde{\theta}_i)\Big)^2\Big].
\end{align}
Now, suppose that we pick a random variable $J$ uniformly at random from $\{1,\dots,d\}$ and collect data only from the arm $\nu_{i,J}$. Then, we define:
\begin{align}
&\mathcal{L}^\texttt{c,sparse}_\texttt{TD}(\theta_i|\coma{i}\shorte 1) = \mathbb{E}_{(s_i, a_i, s_i')\sim\mathcal{D}_{\senderfor{i,J}}}\Big[\Big(R(s_i)-\lambda c_{a_i}\nonumber\\
&+ \beta \max_a Q_i(s_i',a;\theta_i)-Q_i(s_i,a_i;\tilde{\theta}_i)\Big)^2\Big].
\end{align}
Then, we have:
\begin{align}
    \mathcal{L}^\texttt{c,dense}_\texttt{TD}(\theta_i|\coma{i}\shorte 1) &= \mathbb{E}[ \mathcal{L}^\texttt{c,sparse}_\texttt{TD}(\theta_i|\coma{i}\shorte 1)] \nonumber \\
 \nabla\mathcal{L}^\texttt{c,dense}_\texttt{TD}(\theta_i|\coma{i}\shorte 1) &= \mathbb{E}[\nabla  \mathcal{L}^\texttt{c,sparse}_\texttt{TD}(\theta_i|\coma{i}\shorte 1)] \label{eq:stoc_approx}
\end{align}

Now consider the gradient descent (GD) in Equation~\eqref{equ:coms_t} with $\mathcal{L}^\texttt{c,dense}_\texttt{TD}(\theta_i|\coma{i}\shorte 1)$. When this loss is replaced with $\mathbb{E}[ \mathcal{L}^\texttt{c,sparse}_\texttt{TD}(\theta_i|\coma{i}\shorte 1)]$, it becomes stochastic gradient descent (SGD) for $\mathcal{L}^\texttt{c,dense}_\texttt{TD}(\theta_i|\coma{i}\shorte 1)$. SGD is widely deployed to optimize loss functions in deep learning since GD is much more expensive \cite{bottou2010large}. It can be shown that for a large class of loss functions, SGD converges to the same optimal solution as GD \cite{robbins1951stochastic,polyak1992acceleration}. Notice that even in our case, we reduce the computational complexity per-step by a factor of $O(d)$. 

We now analyze the total sample complexity. \cite{chen2019information} proves that, if we have $n$ training samples, with probability at least $1-\delta$, the value function after $k$ iterations satisfies $v^*-v^k\le \epsilon_eV_{max}$, and
\begin{align}
    n=O\left(\frac{C\log\frac{\left|\mathcal{F}\right|}{\delta}}{\epsilon_e^2(1-\beta)^4}\right).
\end{align}
Here, $\mathcal{F}$ is the class of candidate value function. When comparing the sample complexity under $\mu_{\nu_i}$ and $\mu_{\bar{\nu}_i}$, we have the same function class, so the sample complxity depends on the concentratability coefficient. We have
\begin{align}
    \frac{v(s,a)}{\mu_{\nu_i}(s,a)} \ge \frac{v(s,a)}{d\mu_{\bar{\nu}_i}(s,a)}, \ \ \forall (s,a)
\end{align}
for all admissible $v$.

Suppose that
\begin{align}
    C_{\bar{\nu}_i} = \frac{v_b(s_b,a_b)}{\mu_{\bar{\nu}_i}(s_b,a_b)},
\end{align}
we have
\begin{align}
    \frac{1}{d}C_{\bar{\nu}_i} &= \frac{1}{d}\frac{v_b(s_b,a_b)}{\mu_{\bar{\nu}_i}(s_b,a_b)}\\
    &\le \frac{v_b(s_b,a_b)}{\mu_{\nu_i}(s_b,a_b)}\\
    &\le C_{\nu_i}.
\end{align}
Therefore, the sample complexity of the sparse communication scheme is at most $d$ times greater than that of the dense communication scheme, and these two schemes can both learn $\epsilon_e$-optimal value functions.
\end{proof}

\subsection{Proposition~\ref{prop:utor}}\label{appx:proof:useful_com}
\usefulcom*

Let's first consider the case where a noise-free arm $i$ learns without communication, i.e., the arm conducts independent Q-learning. The best known sample-complexity result about model-free Q-learning is provided by~\cite{wang2019q}. Using a UCB exploration strategy, it can be proved that for any $1/2<\beta<1$, $\epsilon>0$, $\delta>0$, with a probability $1-\delta$, the number of timestep $t$ such that the policy is not $\epsilon$-optimal at state $s_t$ is at most
\begin{align}   t_{\texttt{ind}}=\tilde{\mathcal{O}}\left(\frac{|\mathcal{S}||\mathcal{A}|\text{ln}1/\delta}{\epsilon^2(1-\beta)^7}\right).\label{equ:ind_sample_complexity}
\end{align}
Here, $\tilde{\mathcal{O}}$ suppresses logarithmic factors of $1/\epsilon$, $1/(1-\beta)$, and $|\mathcal{S}||\mathcal{A}|$. 

Eq.~\ref{equ:ind_sample_complexity} gives the number of samples needed to get a good estimation of $Q_i$ in the worst case scenario. We now compare it to the case with communication.

Without communication, arm $i$ uses a behavior policy $\pi_{\nu_i}$ to collect samples for the Q update. This falls in the range of asynchronous Q-learning. We know that~\cite{li2020sample} has proven that, for any $0<\delta<1$ and $0<\epsilon <1/(1-\delta)$, with a probability of at least $1-\delta$, the number of timesteps where the Q value is not $\epsilon$-optimal is at most
\begin{align}
t_{\texttt{comm}}=\tilde{\mathcal{O}}\left(\text{ln}\frac{1}{\delta}\frac{1}{\mu_{\texttt{min}}}\left(\frac{1}{\epsilon^2(1-\beta)^5}+\frac{t_\texttt{mix}}{1-\beta}\right)\right).\label{equ:comm_sample_complexity}
\end{align}
Again, $\tilde{\mathcal{O}}$ suppresses logarithmic factors of $1/\epsilon$, $1/(1-\beta)$, and $|\mathcal{S}||\mathcal{A}|$.

\subsection{Noise-free to noisy communication channels}\label{appx:proof:examples}
\ruuseful*
\begin{proof}
\begin{figure}
    \centering
    \begin{subfigure}{\linewidth}
    \centering
    \includegraphics[width=0.8\linewidth]{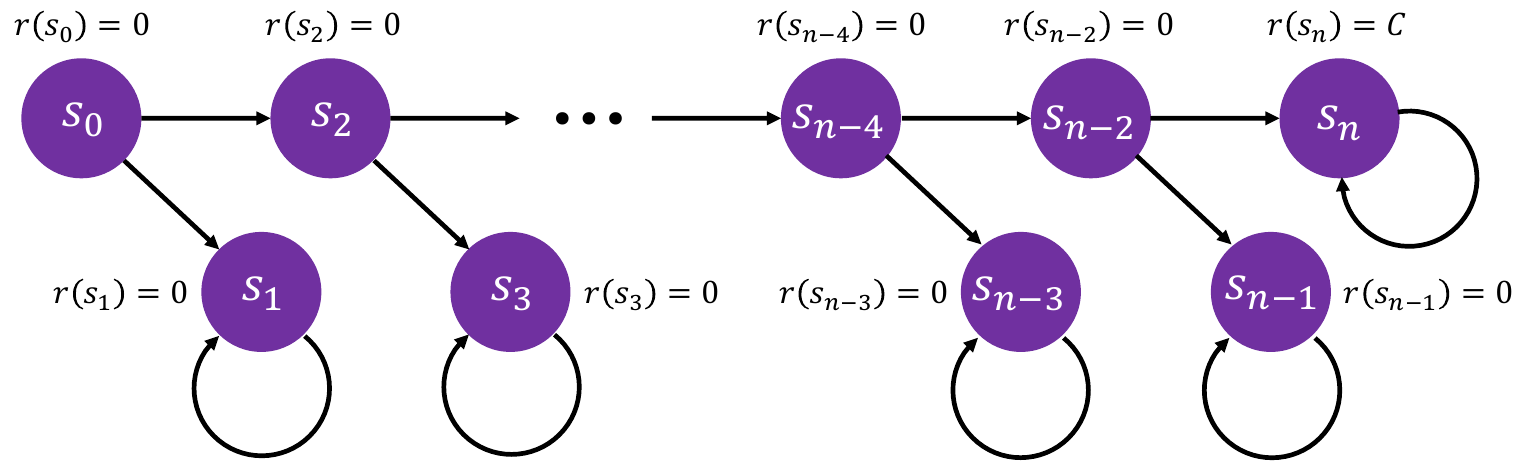}
    \caption{The MDP associated with a noise-free arm.}
    \end{subfigure}
    \begin{subfigure}{\linewidth}
    \centering
    \includegraphics[width=0.8\linewidth]{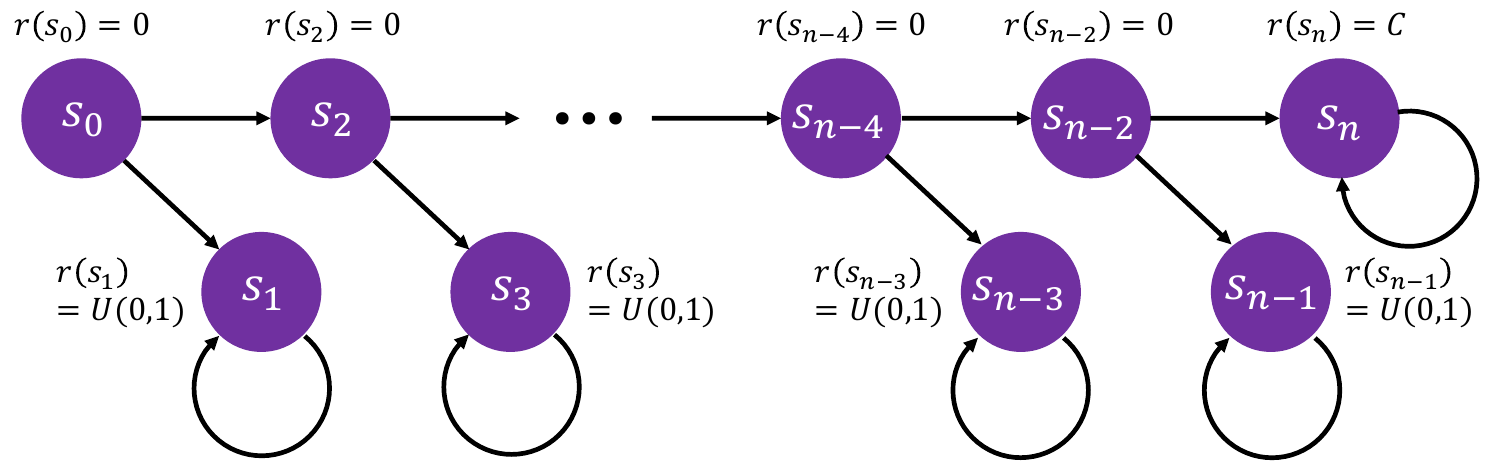}
    \caption{The MDP associated with an noisy arm.}
    \end{subfigure}
    \caption{The example MDP family where communication from a noise-free arm to an noisy arm is helpful.}
    \label{fig:r2u_helpful}
\end{figure}
For any number of $n>0$ and $\beta>0$, consider the following MDP (see also Fig.~\ref{fig:r2u_helpful}). The state set is $\mathcal{S}=\{s_0, s_1, \cdots, s_n\}$, and the action set is $\{a_0,a_1\}$. Taking $a_0$ and $a_1$ at $s_k$, where $k$ is an even number, deterministically lead to state $s_{k+2}$ and $s_{k+1}$, respectively. We assume that $s_{n+2}=s_n$. For state $s_k$ where $k$ is an odd number, both actions leads back to the state itself. In the noise-free arm, the reward is 0 for all states except for state $s_n$, where $R(s_n)=C$. In the noisy arm, all states $s_k$ with an odd $k$ are affected by noise, which is a uniform random number in the range (0,1).

We consider Q-learning with the $\epsilon$-greedy exploration strategy. All Q values are initialized to 0. In cases where the Q values of two actions at a state are the same, we break the tie randomly. We observe that, for a sufficiently large reward at $s_n$ where $R(s_n)=C>\left(\frac{1}{\beta}\right)^{n/2-2}$, the Q-function before visiting $s_n$ cannot be optimal. This is because, in the noisy arm, the best return of visiting an odd-number state is $\beta/(1-\beta)$, while the return of visiting $s_n$ is $C\beta^{n/2-1}/(1-\beta)$.

We explore the environment and update Q functions in $K=O(|\mathcal{S}|)$ epochs. Use $A_k$ to denote the event that $s_n$ is visited in the $k^{\text{th}}$ trajectory, and use $\bar{A}_{1:k-1}$ to denote the event that $s_n$ has not been visited in the first $k-1$ trajectories. We calculate the probability that $s_n$ is not explored in $K$ epochs, denoted by $Pr(\bar{A}_{1:K})$.

In the noisy arm, we have that
\begin{align}\label{equ:noise_not_visited_prob}
    Pr(\bar{A}_{1:K}) = Pr(\bar{A}_{1:1})Pr(\bar{A}_{1:2}|\bar{A}_{1:1})\cdots Pr(\bar{A}_{1:K}|\bar{A}_{1:K-1}).
\end{align}
For the first term, 
\begin{align*}
    Pr(\bar{A}_{1:1}) = 1-P(A_1) = 1-\left(\frac{1}{2}\right)^\frac{n}{2},
\end{align*}
because the only possible case where $A_1$ happens is if $a_0$ is selected at each state.

For the following terms, when $s_n$ is not visited in preceding epochs, there is at least one other state $s_x$, where $x$ is an odd number, that would have been visited. In this case, we have $Q(s_{x-1}, a_1)>Q(s_{x-1}, a_0)$ because $R(s_{x})$ is strictly positive. Therefore, the probability that state $s_{x+1}$ could be visited from $x-1$ decreases to $\epsilon/2$. It follows that
\begin{align*}
    P(\bar{A}_{1:k}|\bar{A}_{1:k-1}) = 1-\left(\frac{1}{2}\right)^{\frac{n}{2}-n_k}\frac{\epsilon}{2}^{n_k} = 1-\epsilon^{n_k}\left(\frac{1}{2}\right)^\frac{n}{2}.
\end{align*}
Here, $n_k$ is the number of odd-number states that have been visited in the first $k-1$ epochs. Since $1\ge n_k\le \frac{n}{2}$, we have that 
\begin{align}\label{equ:noise_not_visited_prob_k}
    1-\epsilon\left(\frac{1}{2}\right)^\frac{n}{2} < P(\bar{A}_{1:k}|\bar{A}_{1:k-1}) < 1-\left(\frac{\epsilon}{2}\right)^\frac{n}{2}
\end{align}
Combining Eq.~\ref{equ:noise_not_visited_prob} and Eq.~\ref{equ:noise_not_visited_prob_k}, we get
\begin{align}\label{equ:noise_not_visited_prob_bound}
& Pr(\bar{A}_{1:K}) \in \Bigg(\left(1-\left(\frac{1}{2}\right)^\frac{n}{2}\right)\left(1-\epsilon\left(\frac{1}{2}\right)^\frac{n}{2}\right)^{K-1}, \nonumber\\
&\left(1-\left(\frac{1}{2}\right)^\frac{n}{2}\right)\left(1-\left(\frac{\epsilon}{2}\right)^\frac{n}{2}\right)^{K-1}\Bigg)
\end{align}
for a noisy arm. Let's look at the upper bound. If we want the probability of finding $s_n$ to be sufficiently large, e.g., larger than $1-a$ where $a$ is a constant, then we need $O\left(\left(\frac{1}{\epsilon}\right)^\frac{n}{2}\right)$ samples. Here we use the approximation $(1-x)^r\approx 1-rx$ for small positive number $x$ and large $r$. This suggests that using an exponentially large number of samples may still result in an exponentially large error in Q estimation, as the estimation of $Q(s_n,a_0)$ is 0 if $s_n$ is not visited, while the true value is $\frac{C}{1-\beta}\ge \left(\frac{1}{\beta}\right)^{n/2-2}\frac{1}{1-\beta}=O\left(\left(\frac{1}{\beta}\right)^{|\mathcal{S}|}\right)$.

In the non-noisy arm, even if an odd number state is visited, the corresponding Q values along the trajectory do not change. Therefore, given that $s_n$ is not explored in the first $k-1$ epochs, the probability of $s_n$ being explored in epoch $k$ is $\frac{1}{2}^\frac{n}{2}$. It follows that, for the noise-free arm, the probability of $s_n$ not being visited in $K$ trajectories is:
\begin{align}
    Pr(\bar{A}_{1:K}) = \left(1-\left(\frac{1}{2}\right)^\frac{n}{2}\right)^K,
\end{align}
which is smaller than the probability in the noisy arm (Eq.~\ref{equ:noise_not_visited_prob_bound}). Therefore, communication of Q value from the non-noisy arm (after visiting $s_n$) can help exploration in the noisy arm, thus accelerating learning in this RMAB.





\end{proof}

\ruuseless*

\begin{proof}
\begin{figure}
    \centering
    \begin{subfigure}{\linewidth}
    \centering
    \includegraphics[width=0.8\linewidth]{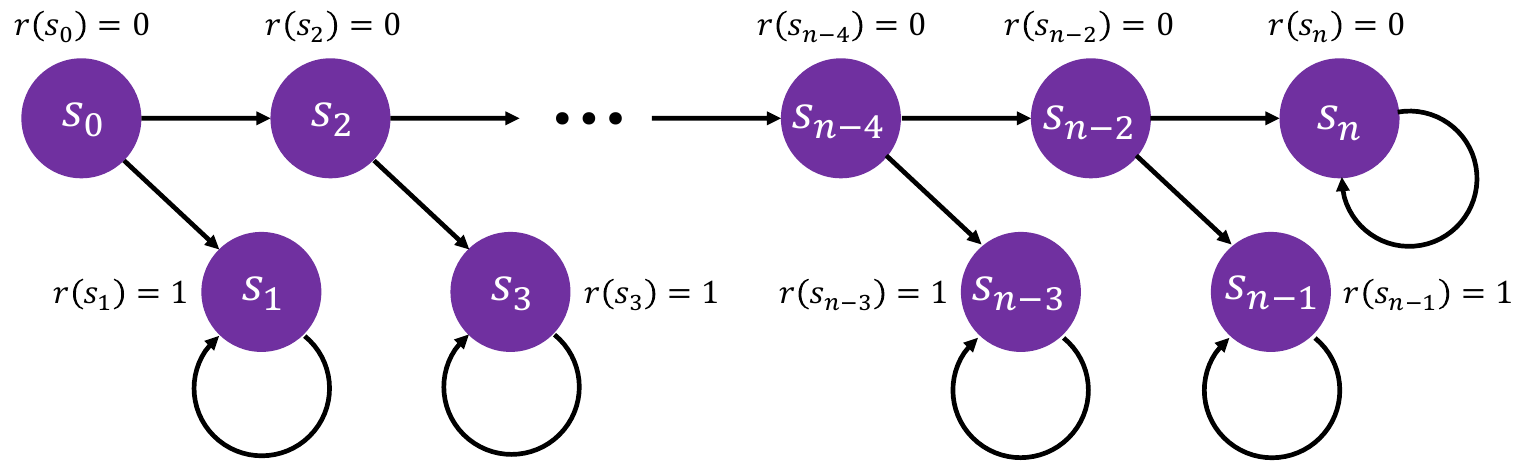}
    \caption{The MDP associated with a noise-free arm.}
    \end{subfigure}
    \begin{subfigure}{\linewidth}
    \centering
    \includegraphics[width=0.8\linewidth]{figs/exp1_noise.pdf}
    \caption{The MDP associated with an noisy arm.}
    \end{subfigure}
    \caption{The example RMAB family where communication from a noise-free arm to an noisy arm is useless.}
    \label{fig:r2u_useless}
\end{figure}
We consider a RMAB setting similar to Lemma~\ref{lemma:r2u_useful}. The only difference is that for all the odd-number states, the reward in the noise-free arm is 1, and the reward at $s_n$ is 0.

In the noisy arm, Eq.~\ref{equ:noise_not_visited_prob_bound} still holds, and the probability of $s_n$ being visited in $K$ trajectories is larger than
\begin{align}
    \left(\frac{1}{2}\right)^\frac{n}{2}+K\left(\frac{\epsilon}{2}\right)^\frac{n}{2}
\end{align}
Here we use the approximation $(1-x)^r\approx 1-rx$ for small positive number $x$ and large $r$.

If we use the Q function of the non-noisy arm to establish the behavior policy, for every $s_k$, where $k$ is an even number smaller than $n$, we have $Q(s_k,a_1)>Q(s_k,a_0)$. Therefore, the probability that $s_n$ is visited in $K$ epochs is
\begin{align}
    O\left(K\left(\frac{\epsilon}{2}\right)^\frac{n}{2}\right).
\end{align}

Comparing these two probabilities gives us the result.
\end{proof}

\section{Additional Related Works}

To the best of knowledge, our work provides the first method to transfer knowledge among agents in RMAB. This setting is different from multi-agent transfer learning in Dec-POMDPs~\cite{qin2022multi,taylor2019parallel}. From the perspective of learning stability, our approaches provides additional robustness for learning in noisy environments, while robust multi-agent learning recently draws great research interests in Dec-POMDPs~\cite{wu2023reward,wu2022model,liu2024efficient,zhang2020robust,li2019robust,guo2022towards,bukharin2024robust,shi2024sample}. This paper also sits within the growing body of work integrating machine learning with game and economic theory and practice, e.g., \cite{hartford2016deep,kuo2020proportionnet,ravindranath2023data, wang2024deep, ravindranath2021deep, tacchetti2019neural, peri2021preferencenet,rahme2021permutation,wang2024multi,wang2024gemnet,dutting2024optimal,zhang2024position,dong2021birds,song2019convergence}. 

RMAB and RL in general have been widely used to model multiple interacting agents \cite{jaques2019social,yu2022surprising,shapley1953stochastic,littman1994markov,jin2021v,xie2020learning,behari2024decision, zhao2023integrating, choromanski2023efficient, seow2024improving, sehanobish2023scalable, boehmer2024optimizing}
For binary action RMABs, Whittle index based Q-learning approaches have been investigated \cite{nakhleh2021neurwin,xiong2023finite,fu2019towards, avrachenkov2022whittle}. Recent works also studied deep RL approaches for multi-action RMABs \cite{killian2021q,zhao2024towards,killian2022restless}. However, existing works on multi-action RMABs largely overlook the challenge of data errors, which is common in real-world scenarios. Our novel communication learning approach enables arms to mitigate the influence of systematic data errors.  

\end{document}